\documentclass{llncs}
\usepackage{amsopn,amssymb,amsmath,mathtools}
\usepackage{color}
\usepackage{comment}
\usepackage{multirow}
\usepackage{graphicx}
\usepackage{float}
\usepackage{subfigure}
\usepackage{epstopdf}
\usepackage{bm}
\usepackage{framed}
\usepackage{url}

\def\F{\mathcal{F}}

\def\R{\mathcal{R}}

\DeclareMathOperator{\bigO}{\mathcal{O}}

\newcommand{\ExpectSub}[2]{\mbox{}{\mathbb{E}}_{#1}\left[#2\right]}
\newcommand{\Expect}[1]{\mbox{}{\mathbb{E}}\left[#1\right]}

\def\reals{\mathbb{R}}

\def\Expectation{\mathbb{E}}
\def\Probability{\mathbb{P}}

\begin{document}
\title{Tensor machines for learning target-specific polynomial features}
\titlerunning{Tensor machines}

\author{Jiyan Yang\inst{1} \and Alex Gittens\inst{2}}
\institute{Stanford University\\
Stanford, CA, USA \\
\email{jiyan@stanford.edu},
\and
International Computer Science Institute\\
Berkeley, CA, USA\\
\email{gittens@icsi.berkeley.edu}
}

\maketitle
\begin{abstract}
Recent years have demonstrated that using random feature maps can significantly decrease the training and testing times of kernel-based algorithms without significantly lowering their accuracy. Regrettably, because random features are target-agnostic, typically thousands of such features are necessary to achieve acceptable accuracies. In this work, we consider the problem of learning a small number of explicit polynomial features. Our approach, named {\it Tensor Machines}, finds a parsimonious set of features by optimizing over the hypothesis class introduced by Kar and Karnick for random feature maps in a {\it target-specific} manner. Exploiting a natural connection between polynomials and tensors, we provide bounds on the generalization error of Tensor Machines. Empirically, Tensor Machines behave favorably on several real-world datasets compared to other state-of-the-art techniques for learning polynomial features, and deliver significantly more parsimonious models.
\end{abstract}

\section{Introduction}
Kernel machines are one of the most popular and widely adopted paradigms in machine learning and data analysis. This success is due to the fact that an appropriately chosen non-linear kernel often succeeds in capturing non-linear structures inherent to the problem without forming the high-dimensional features necessary to explicitly delineate those structures. Unfortunately, the cost of kernel methods scales like $\bigO(n^3)$, where $n$ is the number of datapoints. Because of this high computational cost, it is often the case that kernel methods cannot exploit all the information in large datasets; indeed, even the $\bigO(n^2)$ cost of forming and storing the kernel matrix can be prohibitive on large datasets.

As a consequence, methods for approximately fitting non-linear kernel methods have drawn much attention in recent years.  In particular, explicit random feature maps have been proposed to make large-scale kernel machines practical~\cite{RR07,KK12,PP13,HXGD13,YSAM14,YSFAM14}. The idea behind this method is to randomly choose an $r$-dimensional feature map $\bm{\phi}$ that satisfies $k(\mathbf{x},\mathbf{y}) = \Expect{\bm\phi(\mathbf{x})^T\bm\phi(\mathbf{y})},$ where $k$ is the kernel of interest~\cite{RR07}.  Let $\bm\Phi$ be the matrix whose rows comprise the application of $\bm\phi$ to $n$ datapoints, then the kernel matrix $\mathbf{K}$ has the low-rank approximation $\bm{\Phi\Phi}^T.$ This approximation considerably reduces the costs of both fitting models, because one can work with the $n$-by-$r$ matrix $\bm\Phi$ instead of $\mathbf{K}$, and of forming the input, because one need form only $\bm\Phi$ instead of $\mathbf{K}.$ 

The success of this approximation hinges on choosing $\bm\phi$ so that $\bm\phi(\mathbf{x})^T \bm\phi(\mathbf{y})$ is close to $k(\mathbf{x},\mathbf{y})$ with high probability.  One easy way to accomplish this is to take $\bm\phi(\mathbf{x}) = \frac{1}{\sqrt{r}}(\phi_1(\mathbf{x}), \ldots, \phi_r(\mathbf{x})),$ where the random features $\phi_i$ all satisfy $\Expect{\phi_i(\mathbf{x})\phi_i(\mathbf{y})} = k(\mathbf{x},\mathbf{y}).$ The concentration of measure phenomenon then ensures that $|k(\mathbf{x},\mathbf{y}) - \bm\phi(\mathbf{x})^T\bm\phi(\mathbf{y})|$ is small with high probability.  In practice, one must use a large number of random features to achieve performance comparable with the exact kernel: $r$ must be on the order of tens or even hundreds of thousands~\cite{HXGD13}. The reason for this is simple: since knowledge of the target function is not used in generating the feature map $\bm\phi,$ to ensure good performance enough random features must be selected to get good approximations to \emph{arbitrary} functions in the reproducing kernel Hilbert space associated with $k.$ Thus under the random feature map paradigm, one trades off a massive reduction in the cost of optimization for the necessity of generating a large number of random features, each of which is rather uninformative about any given target.

Implicit in the formulation of random feature maps is the assumption that the chosen class of features is expressive enough to cover the function space of interest, yet simple enough that features can be sampled randomly in an efficient manner.  Given this assumption, it seems natural to question whether it is possible to directly select a much smaller number of features relevant to the given target in a computational efficient manner. 
\begin{framed}
    \centerline{\textbf{Target-specific optimization vs target-agnostic randomization}}
    Given a kernel $k$ and a parametrized family of features $\{\phi_\omega\}_{\omega \in \Omega}$ satisfying $k(\mathbf{x}, \mathbf{y}) = \Expect{\phi_\omega(\mathbf{x})^T\phi_\omega(\mathbf{y})}$ where the expectation is with respect to some distribution on $\Omega,$ can optimization over this hypothesis class to select $r$ features give a parsimonious representation of a known target, with the same or less computational effort and less error than target-agnostic random sampling of $\omega$?
\end{framed}

\begin{figure}
\begin{centering}
\includegraphics[width=0.9\textwidth]{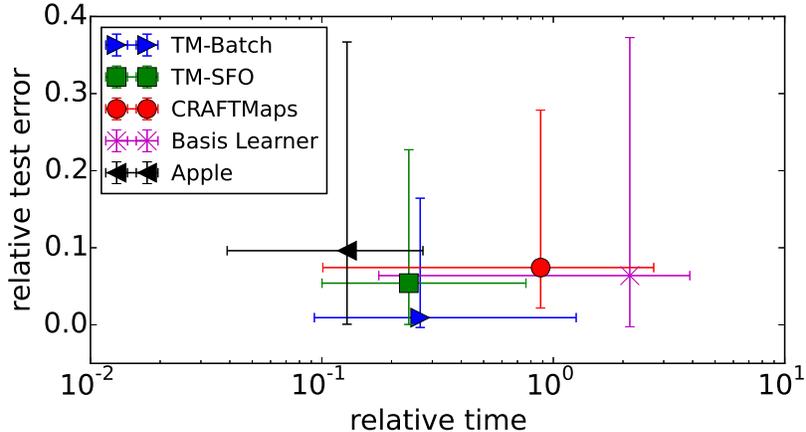}
\caption{ The relative test error and relative running time (relative to kernel ridge regression on a subset of the training data) of several polynomial learning algorithms over $10$ datasets. The median and first and third quartiles of the error and running times are shown. TM-Batch and TM-SFO are solvers for the Tensor Machine model class introduced in this work. See Section~\ref{sxn:empirical} for further discussion.}
\label{fig:overall}
\end{centering}
\end{figure}

This question has recently been answered in the affirmative for the case of the radial basis kernel~\cite{YSSW15,YKRC15}. 
In this work, we provide a positive answer to the above question in the case of polynomial kernels and the Kar--Karnick random features introduced in~\cite{KK12} (see Figure~\ref{fig:overall}). Polynomial kernels are of interest because, after the radial basis kernel, they are arguably the most widely used kernels in machine learning applications. The Weierstrass approximation theorem guarantees than any smooth function can be arbitrarily accurately approximated by a polynomial~\cite{R76}[Chapter 7], so in principle a polynomial kernel of sufficiently high degree can be used to accurately approximate any smooth target function. 

In this work, we use a natural connection between tensors and polynomials to introduce a restricted hypothesis class of polynomials corresponding to low-rank tensors. In the Tensor Machine paradigm, learning consists of learning a regularized low-rank decomposition of a hidden tensor corresponding to the target polynomial. Given $n$ training points in $d$-dimensional input space, the computational cost of fitting a degree $q$ TM with a rank-$r_{\text{TM}}$ approximation for each degree (using a first-order algorithm) is $\bigO(ndq^2r_{\text{TM}})$, while the cost of fitting a predictor using the Kar--Karnick random feature maps and kernel regression approach of~\cite{KK12,HXGD13} is $\bigO(ndqr_{\text{KK}} + nr_{\text{KK}}^2)$. In practice the $r_{\text{KK}}$ required for accurate predictions is typically on the order of thousands; by way of comparison, we show empirically that TMs typically require less than $r_{\text{TM}}=5$ to achieve a comparable approximation error!  

We show experimentally that Tensor Machines strike a favorable balance between expressivity and parsimony, and provide an analysis of Tensor Machines that establishes favorable generalization properties. The objective function for fitting Tensor Machines is nonconvex, but we show empirically that it is sufficiently structured that the process of fitting a TM is robust, efficient, and requires very few features. We demonstrate that our algorithm exhibits a more favorable time--accuracy tradeoff when compared to the random feature map approach to polynomial regression~\cite{KK12,PP13,HXGD13}, as well as to the polynomial network algorithm of~\cite{LSS14a}, and the recently introduced Apple algorithm for online sparse polynomial regression~\cite{ABHLT14}.

\section{Prior Work}

Consider the estimation problem of fitting a polynomial function $f$ to $n$
i.i.d. training points $(\mathbf{x}_i, y_i)$ drawn from the same unknown
distribution, formulated as 
\begin{equation} 
    \label{eqn:polykernelobjective}
    \hat{f} = \text{arg min}_{f \in \mathcal{H}_q} \frac{1}{n} \sum_{i=1}^n \ell(f(\mathbf{x}_i), y_i) + \lambda \|f\|_{\mathcal{H}_q},
\end{equation}
where $\mathcal{H}_q$ is the reproducing kernel Hilbert space of all
polynomials of degree at most $q$ and $\ell$ is a specified loss function.
An exact solution (to within a specified numerical precision) to this
problem can be obtained in $\bigO(n^3)$ time using classical kernel
methods.

One approach in the literature has been to couple kernel methods with various
techniques for approximating the kernel matrix. The underlying assumption is
that the kernel matrix is numerically low-rank, with rank $r \ll n$; typically
these methods reduce the cost of kernel methods from $\bigO(n^3)$ to
$\bigO(nr^2).$ Nystr\"{o}m approximations, sparse greedy approximations, and
incomplete Cholesky factorizations fall into this class~\cite{WS01,SS00,FS02}.
In~\cite{BJ05}, Bach and Jordan observed that prior algorithms in this class
did not exploit \emph{all} the knowledge inherent in supervised learning tasks:
namely, they did not exploit knowledge of the targets (classification labels or
regression values). They showed that by exploiting knowledge of the target, one
can construct low-rank approximations to the kernel matrix that have
significantly smaller rank, with a computational cost that remains
$\bigO(nr^2).$ 

Another approach to polynomial-based supervised learning relies upon the
modeling assumption that the desired target function can be approximated well
in the subspace of $\mathcal{H}_q$ consisting of \emph{sparse} polynomials.
Recall that a sparse degree-$q$ polynomial in $d$ variables is one in which
only a few of the possible monomials have nonzero coefficients (there are
exponentially in $d$ many such monomials). The early work of Sanger et al.
attempts to learn the monomials relevant to the target in an online manner by
augmenting the current polynomial with interaction features~\cite{SSM92}. The
recent Apple algorithm of Agarwal et al. attempts to learn a sparse polynomial,
also in an online manner, using a different heuristic that selects monomials
from the current polynomial to be used in forming the next term of the
monomial~\cite{ABHLT14}.  The algorithms presented in~\cite{APVZ14}
and~\cite{KSDK14} provide theoretical guarantees for fitting sparse
polynomials, but their computational costs scale undesirably for large-scale
learning. 

Polynomial fitting has also been tackled using the neural network paradigm.
In~\cite{LSS14a}, Livni et al. provide an algorithm for learning polynomial
functions as deep neural networks which has the property that the training
error is guaranteed to decrease at each iteration. This algorithm has the
desirable properties that the network learns a linear combination of a set of
polynomials that are constructed in a target-dependent way, and that the
degree of the polynomial does not have to be specified in advance: instead,
additional layers can be added to the network until the desired error threshold
has been reached, with each layer increasing the degree of the predictor by
one. Unfortunately, this algorithm requires careful tuning of the
hyperparameters (number of layers, and the width of each layer). In the
subsequent work~\cite{LSS14b}, Livni et al. provide an algorithm for fitting
cubic polynomials in an iterative manner using rank-one tensor approximations. It
can be shown that, in fact, this algorithm greedily fits cubic polynomials in
the Tensor Machine class we propose in this paper.

Factorization Machines combine the expressivity of polynomial regression with
the ability of factorization models to infer relationships from sparse training
data~\cite{R10}. Quadratic Factorization Machines, as introduced by Rendle, are
models of the form
\[
    f(\mathbf{x}) = \omega_0 + \langle \bm{\omega}_1, \mathbf{x} \rangle + \sum_{i = 1}^d \sum_{j = i+1}^d \langle \mathbf{v}_i, \mathbf{v}_j \rangle x_i x_j,
\]
where a vector $\mathbf{v}_i$ is learned for each coordinate of the input
$\mathbf{x}.$ These models have been applied with great success in
recommendation systems and related applications with sparse input $\mathbf{x}.$
Quadratic FMs can be fit in time linear in the size of the data, the degree of
the polynomial being fit, and the length of the vectors $\mathbf{v}_i$, but can
only represent nonlinear interactions that can be written as symmetric
homogeneous polynomials (plus a constant).  In particular, FMs cannot represent
polynomials that involve monomials containing variables raised to powers
higher than one. For instance, $x_2^2$ cannot be represented as a Factorization
Machine. Another drawback to Factorization Machines is that explicitly evaluating
the sums involved for a degree-$q$ FM requires $\bigO(d^q)$ operations.
A computational manipulation allows quadratic FMs to be fit in linear time,
and it has been claimed that similar manipulations for higher-order FMS~(\cite{R10}),
but no such generalizations have been documented. Perhaps for this reason, only 
second-order FMs have been used in the literature.

The random feature maps approach to polynomial-based learning~\cite{RR07}
exploits the concentration of measure phenomenon to directly construct a random
low-dimensional feature map that approximately decomposes $\mathbf{K}.$ Kar and
Karnick provided the first random feature map approach to polynomial regression
in~\cite{KK12}; their key observation is that the degree-$q$ homogeneous
polynomial kernel $k(\mathbf{x}, \mathbf{y}) = \langle \mathbf{x},
\mathbf{y}\rangle^q$ can be approximated with random features of the form
\[
\phi_i(\mathbf{x}) = \prod_{j=1}^q \langle \bm{\omega}_j^i, \mathbf{x} \rangle,
\]
where the vectors $\bm{\omega}_j^i$ are vectors of random signs; that is,
random feature maps of the form
\[
\bm{\phi} = \frac{1}{\sqrt{r}} (\phi_1, \ldots, \phi_r)
\]
satisfy the condition 
$k(\mathbf{x}, \mathbf{y}) = \Expect{\bm{\phi}(\mathbf{x})^T \bm{\phi}(\mathbf{y}) }$ 
necessary for the random feature map approach. Pagh and Pham provided a
qualitatively different random feature map for polynomial kernel machines,
based on the tensorization of fast hashing transforms~\cite{PP13}. These
TensorSketch feature maps often outperform Kar--Karnick feature maps, but both
methods require a large $r.$ The CRAFTMaps approach of~\cite{HXGD13} combines
either TensorSketch or Kar--Karnick random feature maps with fast
Johnson--Lindenstrauss transforms to significantly reduce the number of features 
needed for good performance.  

Of these methods, the Tensor Machine approach introduced in this paper is most
similar to the Factorization Machine and Kar--Karnick random feature map
approaches to polynomial learning.

\section{Tensors and polynomials} 

To motivate Tensor Machines, which are introduced in the next section, we first
review the connection between polynomials and tensors. For simplicity we
consider only \emph{homogeneous} degree-$q$ polynomials: the monomials of such
a polynomial all have degree $q.$ We show that Kar--Karnick predictors and
Factorizaton Machines correspond to specific tensor decompositions.

Recall that a tensor $\mathbf{T}$ is a multidimensional array, 
\[
\mathbf{T} = (T_{i_1, i_2, \ldots, i_q})_{i_1, \ldots, i_q}.
\]
The number of indices into the array, $q$, is also called the degree of the
tensor. 
The inner product of two conformal tensors $\mathbf{T}$ and $\mathbf{S}$ is
obtained by treating them as two vectors: 
\[
\langle \mathbf{T}, \mathbf{S} \rangle = \sum_{i_1} \sum_{i_2} \cdots \sum_{i_q} T_{i_1, i_2, \ldots, i_q}S_{i_1, i_2, \ldots, i_q}.
\]
The Segre outer product of vectors 
$\bm{\omega}_1 \in \reals^{d_1}, \ldots, \bm{\omega}_q \in \reals^{d_q}$ is the 
$d_1 \times d_2 \times \cdots \times d_q$ tensor that satisfies
\[
(\bm{\omega}_1 \bullet \cdots \bullet \bm{\omega}_q )_{i_1, i_2, \ldots, i_q} = 
(\bm{\omega}_1)_{i_1} (\bm{\omega}_2)_{i_2} \cdots (\bm{\omega}_q)_{i_q}.
\]
The degree-$q$ self-outer product of a vector $\bm{\omega}$ is denoted by 
$\bm{\omega}^{(q)}.$

Given a decomposition for a tensor $\mathbf{T}$ of the form
\[
  \mathbf{T} = \sum_{i=1}^r \bm{\omega}_1^i \bullet \cdots \bullet \bm{\omega}_q^i
\]
where $r$ is minimal, $r$ is called the rank of the tensor. A tensor $\mathbf{T}$ 
is \emph{supersymmetric} if there exists a decomposition of the form
\[
 \mathbf{T} = \sum_{i=1}^r \bm{\omega}_i^{(q)}.
\]



The tensor $\mathbf{x}^{(q)}$ comprises all the degree-$q$ monomials in the
variable $\mathbf{x},$ so any degree-$q$ homogeneous polynomial $f$ satisfies
$f(\mathbf{x}) = \langle \mathbf{T}, \mathbf{x}^{(q)} \rangle$ for some
degree-$q$ tensor $\mathbf{T};$ likewise, any degree-$q$ tensor $\mathbf{T}$
determines a homogenous polynomial $f$ of degree $q.$ This equivalence between
homogeneous polynomials and tensors allows us the attack the problem of
polynomial learning as one of learning a tensor. The Kar--Karnick random
feature maps and Factorization Machine approaches can both be viewed through
this lens. 

A single Kar--Karnick random polynomial feature corresponds to a rank-one tensor:
\begin{multline*}
 \phi(\mathbf{x}) = \prod_{j=1}^q \langle \bm{\omega}_j, \mathbf{x} \rangle = 
 \prod_{j=1}^q \left( \sum_{i=1}^d (\bm{\omega}_j)_i x_i \right) \\
 = \sum_{i_1=1}^d \cdots \sum_{i_q=1}^d (\bm{\omega}_1)_{i_1} \cdots (\bm{\omega}_q)_{i_q} x_{i_1} \cdots x_{i_q} \\
= \langle \bm{\omega}_1 \bullet \cdots \bullet \bm{\omega}_q, \mathbf{x}^{(q)} \rangle.
\end{multline*}
Accordingly, predictors generated by the random feature maps approach with the 
Kar--Karnick feature map $\bm{\phi}$ satisfy 
\[
f(\mathbf{x}) = \sum_{i=1}^r \alpha_i \phi_i(\mathbf{x}) = 
\left\langle \sum_{i=1}^r \alpha_i \bm{\omega}_1^i \bullet \bm{\omega}_2^i \bullet \cdots \bullet \bm{\omega}_q^i, \mathbf{x}^{(q)} \right\rangle.
\]
for some coefficient vector $\bm{\alpha}.$ That is, Kar--Karnick predictors 
correspond to tensors in the span of $r$ \emph{randomly sampled} degree-$q$ rank-one tensors.

Factorization Machines fit predictors of the form
\[
    f(\mathbf{x}) = \sum_{\mathclap{i_q > i_{q-1} > \cdots > i_1}} \langle \mathbf{v}_{i_1}, \ldots, \mathbf{v}_{i_q}\rangle x_{i_1} \cdots x_{i_q},
\]
where the factors $\mathbf{v}_i$ are vectors in $\reals^m$ and 
\[
\langle \mathbf{v}_{i_1}, \ldots \mathbf{v}_{i_q} \rangle = 
\sum_{j=1}^m (\mathbf{v}_{i_1})_j \cdots (\mathbf{v}_{i_q})_j
\]
is the generalization of the inner product of two vectors. Let $\mathbf{V}$ be 
the $d\times m$ matrix with rows comprising the vectors $\mathbf{v}_i,$ and 
define the $d$-dimensional vector $\bm{\omega}_j$ to be the $j$th column of 
$\mathbf{V}.$ It follows that FMs can be expressed in the form
\begin{align*}
    f(\mathbf{x}) & = \sum_{i_1=1}^{d-q+1} \cdots \sum_{i_q = i_{q-1}+1}^d \left( \sum_{j=1}^m (\mathbf{v}_{i_1})_j \cdots (\mathbf{v}_{i_q})_j \right) x_{i_1} \cdots x_{i_q} \\
                  & = \sum_{i_1=1}^{d-q+1} \cdots \sum_{i_q = i_{q-1}+1}^d \left( \sum_{j=1}^m (\bm{\omega}_j)_{i_1} \cdots (\bm{\omega}_j)_{i_q} \right) x_{i_1} \cdots x_{i_q} \\
                  & = \sum_{i_1=1}^{d-q+1} \cdots \sum_{i_q = i_{q-1}+1}^d \left( \sum_{j=1}^m \bm{\omega}_j^{(q)} \right)_{i_1, \ldots, i_q} \left(\mathbf{x}^{(q)} \right)_{i_1, \ldots, i_q} \\
 & = \left\langle \sum_{j=1}^m \bm{\omega}_j^{(q)} - \mathbf{R}, \mathbf{x}^{(q)} \right\rangle,
\end{align*}
where the tensor $\mathbf{R}$ is defined so that its nonzero entries cancel the 
entries of $\sum_{j=1}^m \bm{\omega}_j^{(q)}$ that correspond to coefficients of 
monomials like $x_1x_2^2$ that contain a variable raised to a power larger than one. The vectors $\bm{\omega}_j$ are learned during training (the tensor $\mathbf{R}$ is implicitly defined by these vectors).

\section{Tensor Machines}

Thus, the Kar--Karnick random feature map approach searches for a low-rank
tensor corresponding to the target polynomial and the Factorization Machine
approach attempts to provide a supersymmetric low-rank tensor plus sparse
tensor decomposition of the target polynomial. The Kar--Karnick approach has
the drawback that it tries to find a good approximation in the span of a set of
random rank-one tensors, so many such basis elements must be chosen to ensure
that an unknown target can be represented well. Factorization Machines
circumvent this issue by directly learning the tensor decomposition, but impose
strict constraints on the form of the fitted polynomial that are not
appropriate for general learning applications.

As an alternative to Kar--Karnick random feature maps and factorization
machines, we propose to instead fit Tensor Machines by \emph{directly
optimizing} over rank-one tensors to directly form a low-rank approximation to
the target tensor (i.e., the target polynomial). Since the targets are not, in
general, homogenous polynomials, we learn a different set of rank-one factors
for each degree up to $q$.  

More precisely, Tensor Machines are functions of the form
\[
    f(\mathbf{x}) =  \omega^0 + \langle \bm{\omega}^{1}, \mathbf{x} \rangle +  \sum_{p=2}^q \left\langle \sum_{i=1}^r \bm{\omega}_1^{p,i} \bullet \cdots \bullet \bm{\omega}_p^{p,i}, \mathbf{x}^{(p)} \right\rangle,
\]
obtained as minimizers to the following proxy to the kernel machine objective~\eqref{eqn:polykernelobjective}:
\begin{equation}
\label{eqn:tmobjective}
\frac{1}{n} \sum_{i=1}^n \ell(f(\mathbf{x}_i), y_i) + \lambda \|\bm{\omega}^1\|_2^2 + \lambda \sum_{p=2}^q \sum_{i=1}^r \sum_{j=1}^p \|\bm{\omega}_j^{p,i}\|_2^2.
\end{equation}
By construction, Tensor Machines couple the expressiveness of the Kar--Karnick random feature maps model with the parsimony of Factorization Machines. 

%
%

\section{Generalization Error}

In this section we argue that fitting Tensor Machines using empirical risk minimization makes efficient use of the training data.
We show that the observed risk of Tensor Machines converges to the expected risk at the optimal rate, thus indicating that empirical risk minimization is an efficient method for finding locally optimal Tensor Machines (assuming the optimization method avoids saddle points). For convenience, we consider a variant of Tensor Machines where the norms of the vectors $\bm{\omega}^{p,\ell}_j$ are constrained:
\begin{equation*}
    \begin{aligned}
      & \text{minimize} & & 
        \frac{1}{n} \sum_{i=1}^n \ell\Big( \omega^0 + \langle \bm{\omega}^1, \mathbf{x}_i \rangle +  \sum_{p=2}^q \langle \mathbf{T}_p, \mathbf{x}_i^{(p)} \rangle, y_i\Big) \\
        & \text{subject to} & & 
        \mathbf{T}_p = \sum_{j=1}^r \bm{\omega}^{p,j}_1 \bullet \cdots \bullet \bm{\omega}^{p,j}_p\; \text{for } p \geq 2, \\
        & & & \|\bm{\omega}^1\|_2 \leq B, \\
        & & & \|\bm{\omega}^{p,j}_i\|_2 \leq B \; \text{for all}\; p,j,i.
    \end{aligned}
\end{equation*}
This formulation is not equivalent to the formulation given in~\eqref{eqn:tmobjective}, but by taking $B$ to be the norm of the largest constituent vector in the fitted Tensor Machine, any bound on the generalization error of this formulation applies to the generalization error of Tensor Machines obtained by minimizing~\eqref{eqn:tmobjective}.

Recall that the Rademacher $\R_n(\F)$ complexity of a function class $\F$ measures how well random noise can be approximated using a function from $\F$~\cite{BM02}. 

\begin{definition}[Rademacher complexity]
    Given a function class $\F$ and i.i.d.\ random variables $\mathbf{z}_i$, the Rademacher complexity of $\F$, $\R_n(\F)$, is defined as
  $$ \R_n(\F) = \frac{1}{n} \ExpectSub{\{\mathbf{z}_i\},\sigma}{ \sup_{f\in\F} \sum_{i=1}^n \sigma_i f(\mathbf{z}_i)}, $$
  where the $\sigma_i$ are independent Rademacher random variables (random variables taking values $\pm 1$ with equal probability).
\end{definition}

%
Well-known results (e.g. \cite{BM02}[Theorems 7 and 8]) state that, with high probability over the training data, the observed classification and regression risks are within $\bigO(\R_n(\F) + 1/\sqrt{n})$ of the true classification and regression risks. In fact, the optimal rate of convergence for the observed risk to the expected risk is $\bigO(1/\sqrt{n})$, which can only be achieved when the Rademacher complexity of the hypothesis class is $\bigO(1/\sqrt{n})$ \cite{M08}.

%
%
%
Our main observation is that the Rademacher complexity of a Tensor Machines grows at the rate $\bigO(1/\sqrt{n})$, so the empirically observed estimate of the Tensor Machine risk converges to the expected risk at the optimal rate.

\begin{theorem}
\label{thm:tm-rademacher-complexity}
Let $\F_{d, q,r,B}$ denote the class of degree-$q$, rank-$r$ Tensor Machines on $\reals^d$ with constituent vectors constrained to lie in the ball of radius $B$:
\begin{align*}
    \F_{d, q, r, B} & =  \Big\{f : \mathbf{x} \mapsto \omega^0 + \langle \bm{\omega}^1, \mathbf{x} \rangle + \sum_{p=2}^q \langle \mathbf{T}_p, \mathbf{x}^{(p)} \rangle \;\Big|\\
                    & \quad \quad \mathbf{T}_p = \sum_{j=1}^r \bm{\omega}^{p,j}_1 \bullet \cdots \bullet \bm{\omega}^{p, j}_p \;\text{for}\; p \geq 2, \\
                    & \quad \quad \|\bm{\omega}^1\|_2 \leq B, \; \text{and}\; \|\bm{\omega}^{p,j}_i\|_2 \leq B \;\text{for all}\; p,j,i\Big\}.
\end{align*}
Let $(\mathbf{x}, y)$ be distributed according to $\Probability$, and assume $\|\mathbf{x}\|_2 \leq B_{\mathbf{x}}$ almost surely. The Rademacher complexity, with respect to $(\mathbf{x}, y)$, of this hypothesis class satisfies
\[    
\R_n(\F_{d,q,r,B}) \leq \frac{c r {(1 + 8 B B_\mathbf{x})}^q q^2 (\sqrt{q d \ln d} + \sqrt{d})}{\sqrt{n}} 
\]
where $c$ is a constant.
\end{theorem}

This result follows from reformulating the Rademacher complexity as the spectral norm of a Rademacher sum of data-dependent tensors and then applying a recent bound from~\cite{NDT13} on the spectral norm of random tensors. A full proof is provided in the appendix.

%
%

\section{Empirical Evaluations}
\label{sxn:empirical}

In this section, we evaluate the performance of Tensor Machines on several
real-world regression and classification datasets and demonstrate their
attractiveness in learning polynomial features relative to other
state-of-the-art algorithms.


\subsection{Experimental setup}

The nonconvex optimization problem~\eqref{eqn:tmobjective} is central to
fitting Tensor Machines. We consider two solvers for~\eqref{eqn:tmobjective}.
The first uses the implementation of L-BFGS provided in Mark Schmidt's
\texttt{minFunc}~\footnote{The 2012 release, retrieved from
\url{http://www.cs.ubc.ca/~schmidtm/Software/minFunc.html}} MATLAB package for
optimization~\cite{BNS94}.  Since L-BFGS is a batch optimization algorithm, we
also investigate the use of SFO, a stochastic quasi-Newton solver designed to
work with minibatches~\cite{SPG14}. We use the reference implementation of SFO
provided by
Sohl-Dickstein~\footnote{\url{https://github.com/Sohl-Dickstein/Sum-of-Functions-Optimizer/blob/master/README.md}}.
We refer to the two algorithms used to fit Tensor Machines as TM-Batch and
TM-SFO, respectively.

The choice of initialization strongly influences the performance of both
TM-Batch and TM-SFO. Accordingly, we used initial points sampled from a
$\mathcal{N}(0, \alpha^2)$ distribution. The variance $\alpha^2$ as well as the
regularization parameter $\lambda$ in~\eqref{eqn:tmobjective} are set by
tuning.  

We compared TM-Batch and TM-SFO against several recent algorithms for learning
polynomials: CRAFTMaps~\cite{HXGD13}, Basis Learner~\cite{LSS14a} and
Apple~\cite{ABHLT14}. To provide a baseline, we also used Kernel Ridge
Regression (KRR) to learn polynomials; in cases where the training set contained
more than 40,000 points, we randomly chose a subset of size 40,000 to
perform KRR.  For CRAFTMaps, the up-projection dimensionality is set to be $4$
times the down-projection dimensionality.  We did not obtain reasonable
predictions using the original implementation of Apple in \texttt{vowpal
wabbit}, so we implemented the model in MATLAB and used this implementation to
train and test the model, but reported the running time of the VW
implementation (with the same choice of parameters).  The remaining methods are
implemented in pure MATLAB. The experiments were conducted on a machine with four
6-core Intel Xeon 2 GHz processors and 128 GB RAM.

Because the performance of nonlinear learning algorithms can be very
data-dependent, we tested the methods on a collection of $10$ publicly
available datasets to provide a broad characterization of the behavior of these
algorithms.  A summary of the basic properties of these datasets can be found
in Table~\ref{table:datasets}. For each dataset, the target degree $r$ was
chosen to be the value that minimized the KRR test error.  
We proprocessed the data by first normalizing the input features to have
similar magnitudes --- viz., so that each column of the training matrix has
unit Euclidean norm --- then scaling each datapoint to have unit Euclidean
norm.

For regression tasks, we used the squared loss, and for binary classification
tasks we used the logistic loss.  For regression tasks, the test error is
reported as $\| \hat y - y^\ast \|_2 / \|y^\ast\|_2$ where $\hat y$ and
$y^\ast$ are the prediction and ground truth, respectively; inaccuracy is
reported for classification tasks.

\begin{table}[t]
\begin{center}
\small
\begin{tabular}{|c|c|c|c|c|}
\hline
  Name & Train/Test split & $d$ & Type & $q$\\
\hline
  Indoor & 19937/11114 & 520 & regression & 3 \\
  Year & 463715/51630 & 90 & regression & 2 \\
  Census & 18186/2273 & 119 & regression & 2 \\
  Slice & 42291/10626 & 384 & regression & 5 \\
  Buzz & 466600/116650 & 77 & regression & 3 \\
  Gisette & 6000/1000 & 5000 & binary & 3 \\
  Adult & 32561/16281 & 122 & binary & 3 \\
  Forest & 522910/58102 & 54 & binary & 4 \\
  eBay search & 500000/100000 & 478 & binary & 3 \\
  Cor-rna & 59535/157413 & 8 & binary & 4 \\
\hline
\end{tabular}
\end{center}
\caption{Description of datasets. The target degree $q$ was selected by minimizing the KRR test error on each dataset (or a subset).
}
\label{table:datasets}
\end{table}

Each algorithm is governed by several interacting parameters, but for each
algorithm we identify one major parameter: the number of iterations for
TM-Batch, the number of epochs for TM-SFO, the number of features for
CRAFTMaps, the layer width for Basis Learner, and the expansion parameter for
Apple. We choose
optimal values for the non-major parameters through a 10-fold cross-validation on
the training set.  We varied the major parameter over a wide range 
and recorded the test error and running times for the different values until the performance
of the algorithms saturated. By
saturation, we mean that the improvement in the test error is less than $0.02
\times\text{err}(\texttt{krr})$ where $\text{err}(\texttt{krr})$ denotes the
test error of KRR.  For each combination of parameters, the average
test errors and running times over $3$ independent trials are reported.



\subsection{Overall performance} 
To compare their performance, we applied TM-Batch, TM-SFO, CRAFTMaps, Basis Learner, and
Apple to the $10$ datasets listed in Table~\ref{table:datasets}.  We did not
evaluate Factorization Machines because efficient algorithms for fitting FM
models are only available in the literature for the case $q=2$.  It is unclear
whether FMs with higher values of $q$ can be fit efficiently. 

We present
the computation/prediction accuracy tradeoffs of the considered algorithms in Figure~\ref{fig:overall}.
The data plotted are the test errors and running times of the algorithms relative to those of KRR:
$$ \text{relerr} = (\text{err}(\texttt{alg}) - \text{err}(\texttt{krr}))/\text{err}(\texttt{krr}) $$
and
$$ \text{reltime} = \text{time}(\texttt{alg})/\text{time}(\texttt{krr}). $$
The median values of \text{relerr} and \text{reltime} as well as the corresponding first and third quartiles are shown.

The detailed results can be found in Table~\ref{table:full_results}. 
The performance of TM-Batch is consistent across the datasets in the sense that it provides reliable predictions in a fairly short amount of time.  TM-SFO converges to lower quality solutions than TM-Batch, but has lower median and variance in runtime. CRAFTMaps and Basis Learner take significantly more time to yield solutions almost as accurate as the two Tensor Machine algorithms. As one might expect, due to its greedy nature, Apple is the fastest of the algorithms (the time of the \texttt{vowpal wabbit} implementation is reported), but of all the algorithms, TM-Batch and TM-SFO deliver the lowest worst-case relative errors.

\begin{table*}[h]
 \begin{center}
 \scriptsize
 \begin{tabular}{c||c||ccccc}
   Name & KRR & TM-Batch & TM-SFO & CRAFTMaps & Basis Learner & Apple \\
 \hline
   Indoor & 0.00961/50 & 0.00802/3.2(4) & 0.0182/3.9(4) & 0.00559/2.7(300) & {\bf 0.00241}/97 & 0.00712/6.5 \\ 
   Year & 0.00485/310 & 0.00565/85(5) & {\bf 0.00484}/74(5) & 0.00494/39(200) & 0.00496/67 & 0.00489/60 \\
   Census &  0.0667/46 & 0.0774/12(5) &  0.0802/45(5) &  {\bf 0.0767}/7.3(1000) & 0.0788/8.1 & 0.0950/7.3\\
   Slice &  0.0118/441  & 0.0229/80(3) &  {\bf 0.0213}/73(3) &  0.0465/707(7000) & 0.0501/1034 &  0.0788/46 \\
   Buzz &  0.407/624  & 0.407/568(2) &  0.409/440(2) &  0.408/2105(1200) & {\bf 0.373}/2472 &  0.496/325\\
   Gisette &  0.027/6.6 & 0.0266/16(4) &  0.0254/29(4) &  0.0300/63(11500) & 0.0270/62 &  {\bf 0.0240}/32\\
   Adult &  0.150/182  & {\bf 0.149}/0.9(4) &  0.151/2.0(4) &  0.154/17(700) & 0.149/ 30 &  0.150/0.8 \\
   Forest & 0.148/361 & 0.184/657(5) & {\bf 0.178}/569(5) & 0.196/1023(750) & 0.180/3494 & 0.219/14\\
   eBay search & 0.197/446 & 0.197/612(5) & {\bf 0.192}/349(5) & 0.281/1054(800) &  0.281/1642 &  0.269/122 \\
   Cor-rna & 0.0446/423 & {\bf 0.0453}/8.4(3) & 0.0489/7.2(3) & 0.0462/19(200) & 0.0493/1.7 & 0.0489/4.2 \\
 \hline
 \end{tabular}
 \end{center}
 \caption{Test error/running time of each methods on $10$ datasets. The rank parameter $r$ used in TM-Batch and TM-SFO and number of features used in CRAFTMaps is listed in the parenthesis. For Apple, the test error is computed based on a MATLAB implementation of Apple but the running time is recorded by using the \texttt{vowpal wabbit} framework with the same choice of parameters. }
 \label{table:full_results}
 \end{table*}
Both TM solvers required less than $r = 5$ rank-one TM features for each individual degree fit for \emph{all} datasets. Since only one dataset benefited from a quintic fit, the TM models required at most $1 + d + \sum_{p=2}^q prd \leq 72d$ parameters to fit each dataset; this should be compared to the CRAFTMaps models, which required at least $400d \leq qrd$ parameters to fit each dataset, generally produced fits with higher error than the TM solvers, and required longer solution times.


\subsection{Effect of rank parameter $r$}
To investigate the effect of the rank parameter $r$ on the test error of Tensor Machines,
we evaluated both TM-Batch and TM-SFO on the Census and Slice datasets.
On Census, since the target degree is $2$, we evaluate Factorization Machines (FM) as well.
The results are shown in Figure~\ref{fig:krr}.

As expected, increasing $r$ leads to smaller test errors. On Slice, where the target degree $q=5$, the gap between the performance of TMs and KRR is relatively large (almost a factor of 2), while on Census, where $q=2$, the gap is slighter.Interestingly, on Census, increasing rank does not lead to a higher accuracy after $r\geq 3$.
We also observe on Census that the error of TMs is lower than that of FMs; this behaviour was also observed on the other datasets when $q=2.$




\begin{figure}
\begin{centering}
\begin{tabular}{cc}
\subfigure[Census]{
\includegraphics[width=0.48\textwidth]{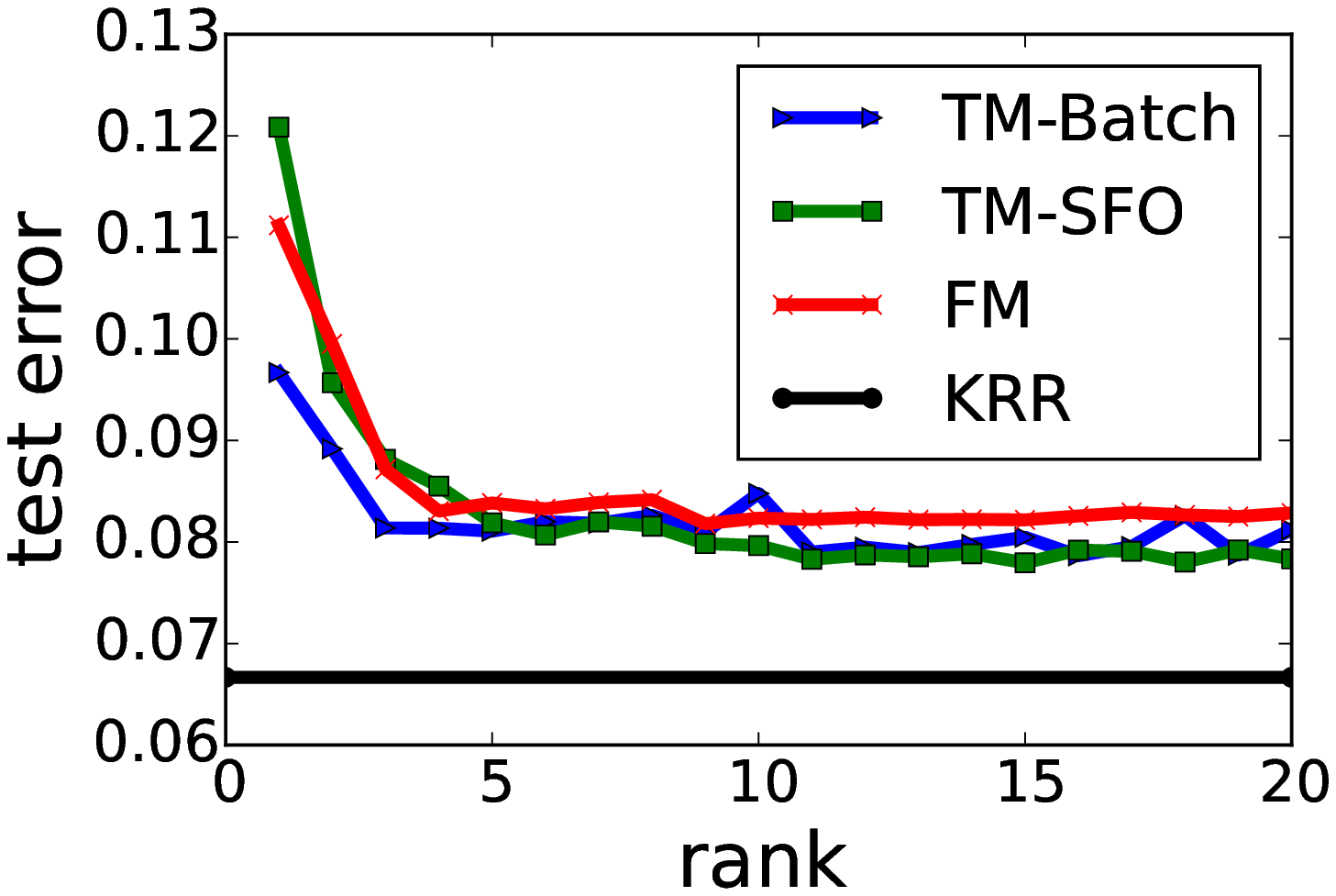}
}
&
\subfigure[Slice]{
\includegraphics[width=0.48\textwidth]{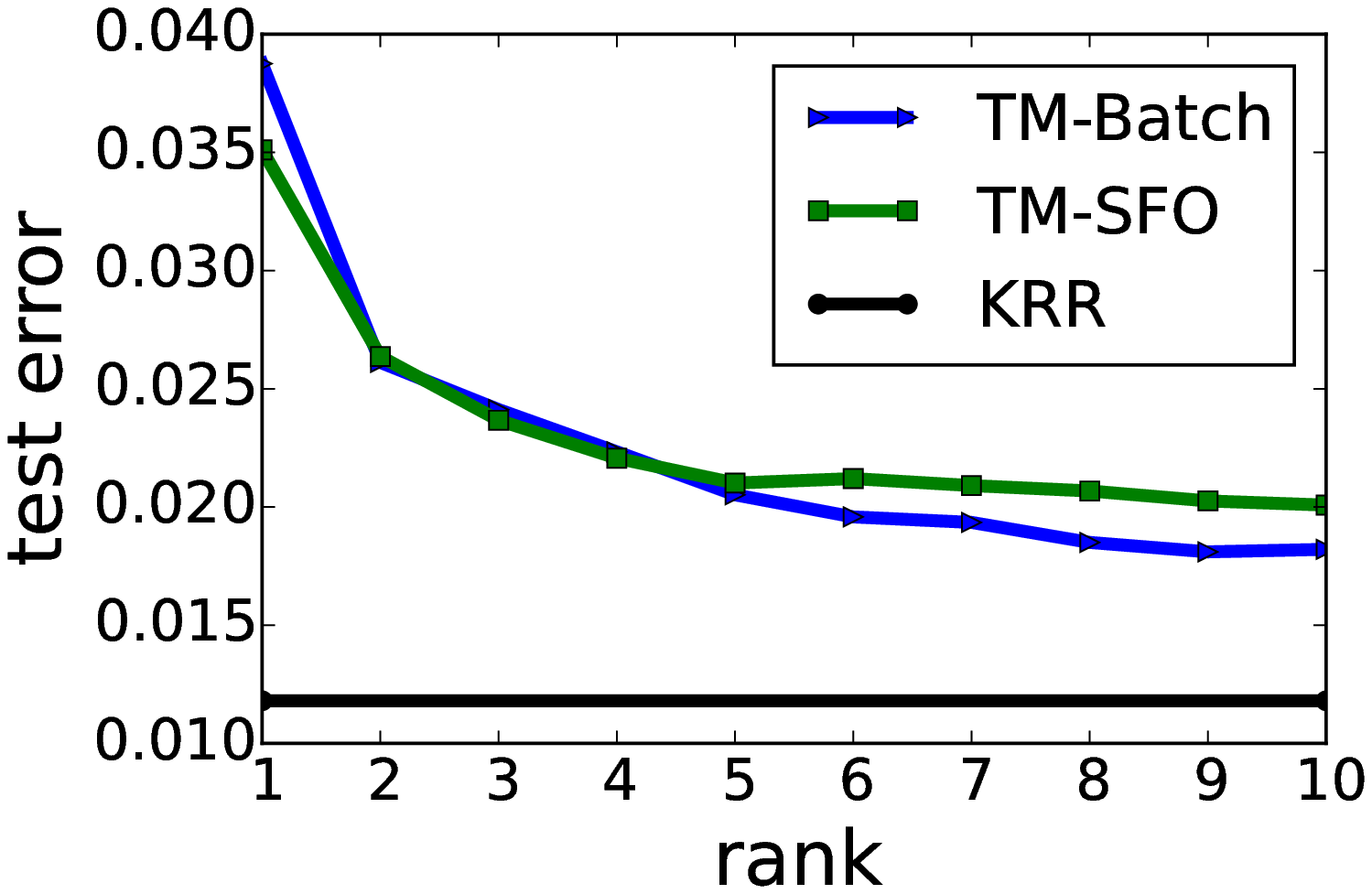}
}
\end{tabular}
\caption{Test error of TM-Batch and TM-SFO with different rank parameters $r$ on the Census ($q=2$) and Slice ($q=5$) datasets.
  The test error of using kernel ridge regression (KRR) is also plotted.
 On Census, we also evaluate Factorization Machines (FM).
 }
\label{fig:krr}
\end{centering}
\end{figure}


\subsection{Time-accuracy tradeoffs}

To assess the time-accuracy tradeoffs of the various algorithms, in Figure~\ref{fig:time_err} we plot the test error vs the training time for the Slice and Forest datasets. The training time is determined by varying the settings of each algorithm's major parameter.

TM-Batch and TM-SFO compare favorably against the other methods: they either produce a much lower error than the other methods (on Slice) 
or reach a considerably low error much faster (on Forest). Similar patterns were observed on most of the datasets we considered.

\begin{figure}
\begin{centering}
\begin{tabular}{cc}
\subfigure[Slice]{
\includegraphics[width=0.48\textwidth]{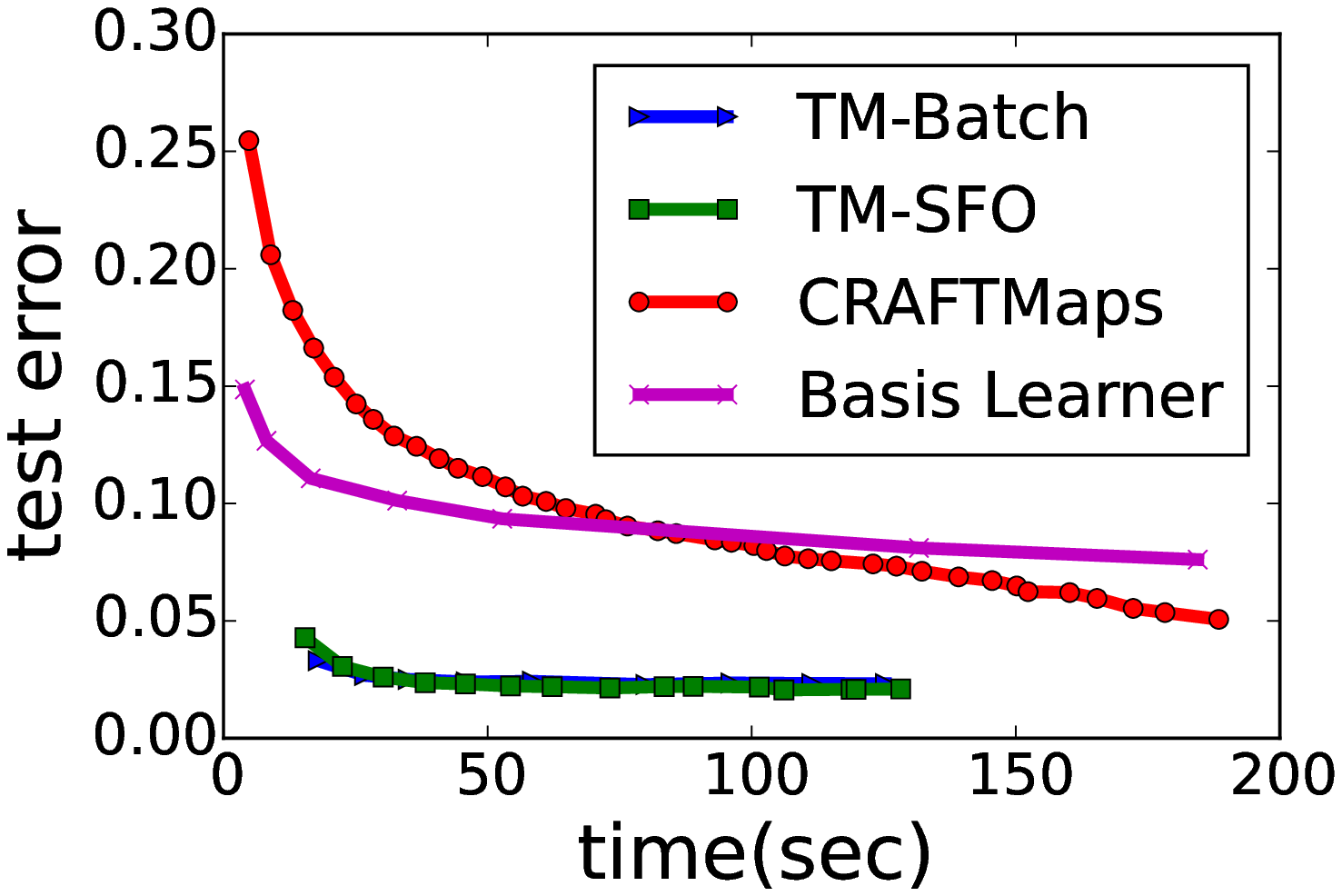}
}
&
\subfigure[Forest]{
\includegraphics[width=0.48\textwidth]{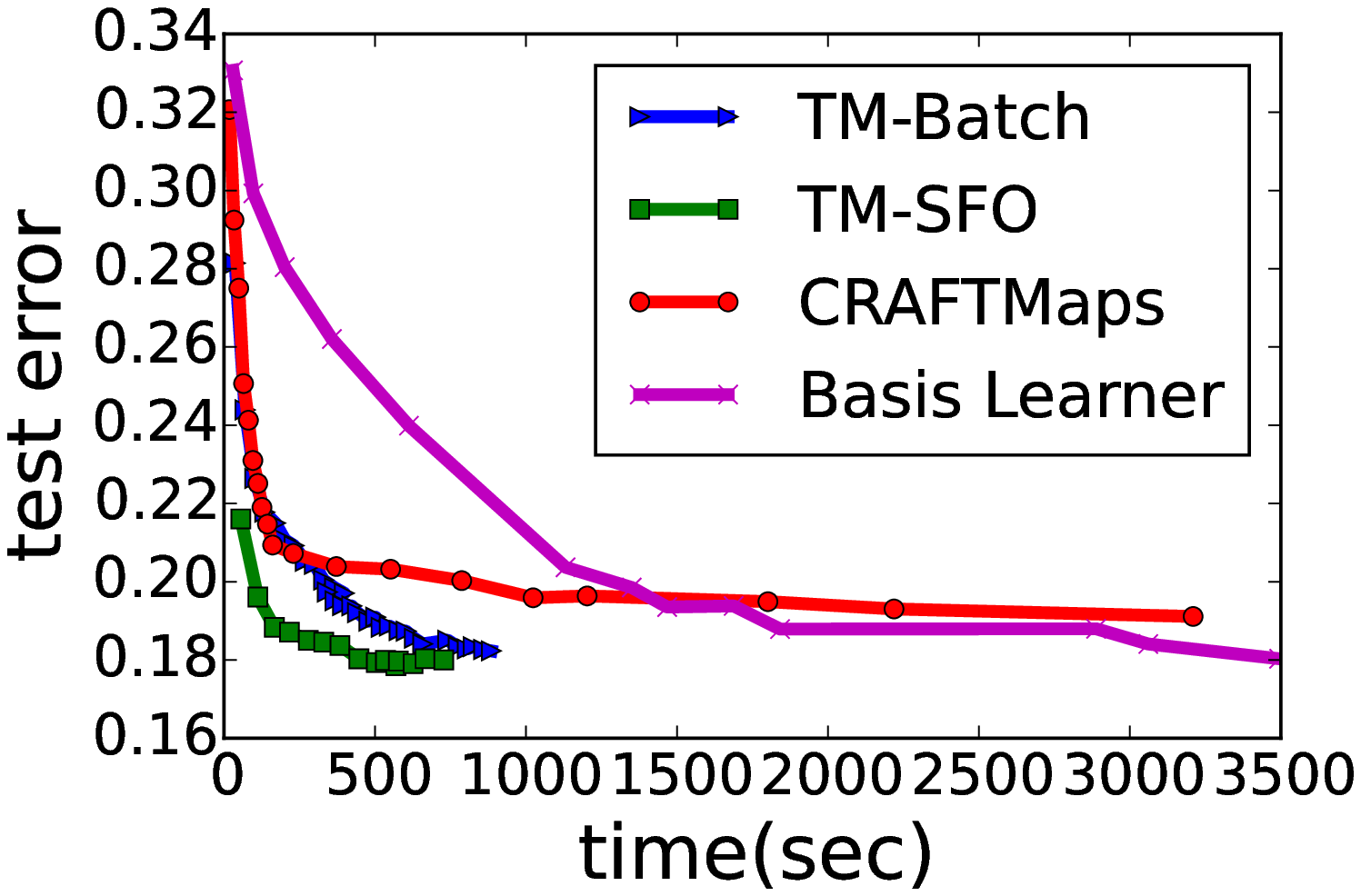}
}
\end{tabular}
\caption{ Evaluation of time-accuracy tradeoffs of the algorithms on the Slice ($q=5)$ and Forest ($q=4$) datasets.
The rank parameters $r$ for TM-Batch and TM-SFO on Slice and Forest are $3$ and $5$, respectively.
}
\label{fig:time_err}
\end{centering}
\end{figure}


\subsection{Scalability of TM-SFO on eBay search dataset}

Since SFO needs to access only a small mini-batch of data points per update, it
is suitable for fitting TMs on datasets which cannot fit in memory all at once.
Here we explore the scalability of TM-SFO on a private eBay search dataset of
1.2 million data points with dimension $478$. Each data point $(u_a,u_b)$
comprises the feature vectors associated with a pair of items that were
returned as the results of a search on the eBay website and were subsequently
visited. The goal is to fit a polynomial for the task of classifying which item
was clicked on first: $a$ or $b$. In our experiments, we randomly selected
100,000 data points to be the test set; training sets of variable size were
selected from the remainder.
 
For comparison, we also evaluate CRAFTMaps on the same task. In TM-SFO, we fix
the rank parameter $r$ to be $5$ and number of epochs to be $50$. In CRAFTMaps,
we fix the number of random features to be $800$. These parameters are the
optimal settings from the experiments used to generate
Table~\ref{table:full_results}. We report both the classification error and
training time as the size of training set grows. 

\begin{figure}
\begin{centering}
\includegraphics[width=0.9\textwidth]{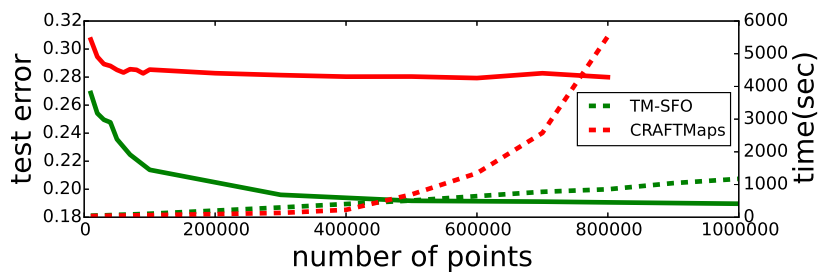}
\caption{Scalability of TM-SFO on eBay search dataset ($q=3$, $r=5$).
The number of epochs is set to be $50$. The plot shows the test error (black
solid line) and training time (red dash line) as a function of the number of
training point used.} 
\label{fig:ebay} 
\end{centering} 
\end{figure}

The results are shown in Figure~\ref{fig:ebay}. We observe that the running
time of TM-SFO grows almost linearly with the size of the training set, while
that of CRAFTMaps grows superlinearly. Also, as expected, because CRAFTMaps
choose the hypothesis space independently of the target, increasing the size of
the training set without also increasing the size of the model does not give
much gain in performance past 100,000 training points. We see that the
target-adaptive nature of TMs endows TM-SFO with two advantages. First, for a
fixed amount of training data the errors of TMs are significantly lower than
those of CRAFTMaps. Second, because the hypothesis space evolves as a function
of the target and training data, the training error of TM-SFO exhibits
noticeable decay up to a training set size of about 500,000 points, long past the point
where the CRAFTMaps test error has saturated.

\bibliographystyle{plain}
\bibliography{tensor_feature}

\appendix

\section{Proof of Theorem~\ref{thm:tm-rademacher-complexity}}

Theorem~\ref{thm:tm-rademacher-complexity} is a corollary of the following bound on the Rademacher complexity of a rank-one Tensor Machine.

\setcounter{theorem}{1}
\begin{theorem}
    \label{thm:simple-tm-complexity}
    Let $\F_{d,q,B}$ denote the class of degree-$q$ rank-one Tensor Machines on $\reals^d$ with constituent vectors constrained to lie in the ball of radius $B$:
    \[
        \F_{d,q,B} = \Big\{ f : \mathbf{x} \mapsto \Big \langle \bm{\omega}_1 \bullet \cdots \bullet \bm{\omega}_q, \mathbf{x}^{(q)} \Big \rangle\;|\; \|\bm{\omega}_j\|_2 \leq B \;\text{for}\; j=1,\ldots,q. \Big\}
    \]
Let $(\mathbf{x}, y)$ be distributed according to $\Probability$, and assume $\|\mathbf{x}\|_2 \leq B_{\mathbf{x}}$ almost surely. The Rademacher complexity with respect to $(\mathbf{x}, y)$ of this hypothesis class satisfies
\[
    \R_n(\F_{d,q,B}) \leq \frac{c {(8 B B_\mathbf{x})}^q q(\sqrt{qd \ln d} + \sqrt{d})}{\sqrt{n}},
\]
where $c$ is a constant.
\end{theorem}

\begin{proof}[Proof of Theorem~\ref{thm:tm-rademacher-complexity}]
    Every function in $\F_{d,q,r,B}$ can be written as the sum of a constant, a linear function of the form $\mathbf{x} \mapsto \langle \bm{\omega}, \mathbf{x}\rangle$ with $\|\bm{\omega}\| < B$, and $r$ functions from each of $\F_{d,2,B}, \ldots, \F_{d,q,B}.$ The Rademacher complexity of $\reals$ is zero, and a straightforward calculation  establishes that the Rademacher complexity of linear functions of the specified form is no larger than $B B_{\mathbf{x}}/\sqrt{n}.$ It follows from a simple structural result on Rademacher complexities~\cite{BM02}[Theorem 12] that
    \[
        \R_n(\F) \leq \frac{BB_{\mathbf{x}}}{\sqrt{n}} + s \sum_{p=2}^q \R_n(\F_{d,p,B}).
        \]
Applying Theorem~\ref{thm:simple-tm-complexity}, we have that
\begin{align*}
    \R_n(\F) & \leq \frac{BB_{\mathbf{x}}}{\sqrt{n}} + cr \sqrt{\frac{d \ln d}{n}} \sum_{p=2}^q (8 B B_{\mathbf{x}})^p p^{3/2} + cr \sqrt{\frac{d}{n}} \sum_{p=2}^q (8BB_{\mathbf{x}})^p p \\
             & \leq cr \sqrt{\frac{d \ln d}{n}} (1 + 8 BB_{\mathbf{x}})^q \sum_{p=1}^q p^{3/2} + cr \sqrt{\frac{d}{n}} (1 + 8 BB_{\mathbf{x}})^q \sum_{p=1}^q p  \\
             & \leq \frac{cr (1 + 8 BB_{\mathbf{x}})^q q^2 ( \sqrt{q d \ln d} + \sqrt{d})}{\sqrt{n}}.
\end{align*}

\end{proof}

\begin{proof}[Proof of Theorem~\ref{thm:simple-tm-complexity}]
    Given the data points $\mathbf{x}_1,\ldots,\mathbf{x}_n$, let 
    $$  \hat \R_n(\F) = \frac{1}{n} \ExpectSub{\sigma}{ \sup_{f\in\F} \sum_{i=1}^n \sigma_i f(\mathbf{x}_i)}$$
  so that $\R_n(\F) = \Expectation_{\{\mathbf{x}_1, \ldots, \mathbf{x}_n\}} \hat \R_n(\F).$
  From the definition of $\F_{d,q,B},$ we have that 
  \begin{align*}
      \hat \R_n(\F) & = \frac{1}{n} \ExpectSub{\sigma}{ \sup_{\bm{\omega}_1, \ldots, \bm{\omega}_q \in \Omega_B} \sum_{i=1}^n \sigma_i \langle \bm{\omega}_1 \bullet \cdots \bullet \bm{\omega}_q, \mathbf{x}_i^{(q)} \rangle} \\
                    & = \frac{1}{n} \ExpectSub{\sigma}{ \sup_{\bm{\omega}_1, \ldots, \bm{\omega}_q \in \Omega_B} \Big \langle \bm{\omega}_1 \bullet \cdots \bullet \bm{\omega}_q, \sum_{i=1}^n \sigma_i \mathbf{x}_i^{(q)} \Big \rangle }
   \end{align*}
   where  $ \Omega_B = \{ \bm{\omega} \in \reals^d \;\vert\; \|\bm{\omega}\|_2 \leq B \} $. 
   
   Define $\mathbf{T}_{\sigma} = \sum_{i=1}^n \sigma_i \mathbf{x}_i^{(q)}$.
   Recall the definition of the spectral norm of an order-$q$ tensor $\mathbf{T}:$
   $$ \|\mathbf{T}\| = \sup_{\mathbf{v}_1,\ldots,\mathbf{v}_q\in \Omega_1} \lvert \langle \mathbf{T}, \bm{v}_1 \bullet \cdots \bullet \bm{v}_q \rangle \lvert. $$
   It follows from this definition that 
  \begin{align}
  \label{eqn:empirical-rademacher-estimate}
  \hat \R_n(\F) & \leq \frac{1}{n} \ExpectSub{\sigma}{ B^q \cdot \sup_{\bm{\omega}_1, \ldots, \bm{\omega}_q \in \Omega_B} \Big \langle \mathbf{T}_\sigma, \frac{\bm{\omega}_1}{\|\bm{\omega}_1\|_2} \bullet \cdots \bullet \frac{\bm{\omega}_q}{\|\bm{\omega}_q\|_2} \Big \rangle} \notag\\
  & = \frac{1}{n} B^q \ExpectSub{\sigma}{\|\mathbf{T}_\sigma\|}.
  \end{align}
  
  For simplicity, we drop the subscript $\sigma$ from $\mathbf{T}_\sigma$ in the following. Note that $\Expect{\mathbf{T}} = 0$, since Rademacher variables are mean-zero. We now apply Theorem~2 of~\cite{NDT13}, which bounds the expected spectral norm of the difference between a tensor Rademacher sum and its expectation with the sum of the expected maximum Euclidean lengths of one-dimensional slices through the tensor:
  \begin{align*}
      \ExpectSub{\sigma}{\|\mathbf{T}\|} & = \ExpectSub{\sigma}{\|\mathbf{T} - \Expectation\mathbf{T}\|} \\
                                         & \leq c 8^q (\sqrt{q \ln d} + 1) \left( \sum\nolimits_{j=1}^q \ExpectSub{\sigma}{\max_{i_1, \ldots, i_{j-1}, i_{j+1}, \ldots, i_q} \left( \sum\nolimits_{i_j=1}^d T_{i_1,\ldots,i_q}^2 \right)^{\tfrac{1}{2}}}\right).
  \end{align*}
  We replace the innermost sum with a maximum to obtain an estimate involving the expected size of the largest entry in $\mathbf{T}:$
\begin{align}
    \label{eqn:tensor-norm-estimate}
    \ExpectSub{\sigma}{\|\mathbf{T}\|} & \leq c8^q (\sqrt{q \ln d} + 1) \left(\sum\nolimits_{j=1}^q \ExpectSub{\sigma}{\max_{i_1, \ldots, i_{j-1}, i_{j+1}, \ldots, i_q} (d \max_{i_j} T_{i_1,\ldots,i_q}^2)^{\tfrac{1}{2}}} \right) \notag\\
                                       & = c8^qq(\sqrt{q d \ln d} + \sqrt{d}) \ExpectSub{\sigma}{\max_{i_1,\ldots,i_q} |T_{i_1, \ldots, i_q}|}.
\end{align}
Next we bound the maximum entry of $\mathbf{T}$ with the Frobenius norm of $\mathbf{T},$ and apply Jensen's inequality to obtain
\begin{align*}
    \ExpectSub{\sigma}{\max_{i_1, \ldots, i_q} |T_{i_1, \ldots, i_q}|} & \leq \ExpectSub{\sigma}{ \langle \mathbf{T}, \mathbf{T} \rangle^{\tfrac{1}{2}} } \leq \big( \ExpectSub{\sigma}{\langle \mathbf{T}, \mathbf{T} \rangle} \big)^{\tfrac{1}{2}} \\
                                                                       & = \left(\sum_{i_1,\ldots,i_q=1}^d \ExpectSub{\sigma}{T_{i_1, \ldots, i_q}^2 } \right)^{\tfrac{1}{2}}.
\end{align*}
Recalling the definition of $\mathbf{T},$ we have that
\begin{align*}
    \left(\ExpectSub{\sigma}{\max_{i_1, \ldots, i_q} |T_{i_1, \ldots, i_q}|}\right)^2 & \leq \sum\nolimits_{i_1, \ldots, i_q=1}^d \ExpectSub{\sigma}{\left( \sum\nolimits_{i=1}^n \sigma_i \mathbf{x}_i^{(q)} \right)_{i_1, \ldots, i_q}^2} \\
                                                                       & = \sum_{i_1,\ldots,i_q=1}^d \sum_{i,j=1}^n \ExpectSub{\sigma}{\sigma_i \sigma_j} \big(\mathbf{x}^{(q)}_i\big)_{i_1, \ldots, i_q}\big(\mathbf{x}^{(q)}_j\big)_{i_1, \ldots, i_q} \\
                                                                       & = \sum_{i_1, \ldots, i_q=1}^d \sum_{i=1}^n \big(\mathbf{x}^{(q)}_i\big)_{i_1,\ldots,i_q}^2 = \sum_{i=1}^n \langle \mathbf{x}^{(q)}_i , \mathbf{x}^{(q)}_i \rangle. 
\end{align*}
It is readily established that 
\[
    \langle \mathbf{v}_1 \bullet \cdots \bullet \mathbf{v}_q, \mathbf{v}_1 \bullet \cdots \bullet \mathbf{v}_q \rangle = \|\mathbf{v}_1\|^2 \cdots \|\mathbf{v}_q\|^2
\]
for any rank-one tensor, so it follows that
\begin{equation}
    \label{eqn:tensor-max-entry-estimate}
    \ExpectSub{\sigma}{\max_{i_1, \ldots, i_q} |T_{i_1, \ldots, i_q}|} \leq \sqrt{\sum_{i=1}^n \|\mathbf{x}_i\|^{2q}} \leq \sqrt{n} B_\mathbf{x}^q.
\end{equation}

From equations~\eqref{eqn:empirical-rademacher-estimate},~\eqref{eqn:tensor-norm-estimate}, and~\eqref{eqn:tensor-max-entry-estimate}, we conclude that
\[
    \R_n(\F) = \ExpectSub{\mathbf{x}_1, \ldots, \mathbf{x}_n}{ \hat \R_n(\F)} \leq \frac{c (8 B B_{\mathbf{x}})^q q (\sqrt{q d \ln d} + \sqrt{d})}{\sqrt{n}}.
\]
\end{proof}

\end{document}